\documentclass[conference]{IEEEtran}
\IEEEoverridecommandlockouts
\usepackage{cite}
\usepackage{algorithm, algorithmic}
\usepackage{graphicx}
\usepackage{textcomp}
\usepackage{xcolor}
\def\BibTeX{{\rm B\kern-.05em{\sc i\kern-.025em b}\kern-.08em
    T\kern-.1667em\lower.7ex\hbox{E}\kern-.125emX}}

\usepackage{amsthm, amsmath, amssymb, amsfonts, graphicx, epsfig, mathtools}
\usepackage{accents}
\usepackage{bbm} 
\usepackage{mathrsfs} 
\usepackage[colorlinks=true, pdfstartview=FitV, linkcolor=blue,
            citecolor=blue, urlcolor=blue]{hyperref}
\usepackage{todonotes}

\newtheorem{theorem}{Theorem}

\newtheorem{lemma}[theorem]{Lemma}

\newtheorem{assumption}[theorem]{Assumption}
\theoremstyle{definition}

\theoremstyle{remark}

\numberwithin{equation}{section}
\numberwithin{theorem}{section}

\def\cF{\mathcal{F}}

\def\cM{\mathcal{M}}

\def\cP{\mathcal{P}}

\def\bA{\mathbb{A}}

\def\bE{\mathbb{E}}

\def\bN{\mathbb{N}}

\def\bP{\mathbb{P}}

\def\bR{\mathbb{R}}

\def\bX{\mathbb{X}}

\def\sF{\mathscr{F}}
\def\sG{\mathscr{G}}

\def\sU{\mathscr{U}}

\def\fZ{\mathfrak{Z}}

\newcommand{\1}{\mathbbm{1}}            
\newcommand{\set}[1]{\{#1\}}            

\DeclareMathOperator{\dif}{d \!}        

\DeclareMathOperator*{\argmin}{arg\,min} 

\DeclareMathOperator{\avar}{\mathrm{AV}@\mathrm{R}}         

\begin{document}

\title{Distributional Method for Risk Averse Reinforcement Learning
\thanks{SJ would like to acknowledge support from the Natural Sciences and Engineering Research Council of Canada (grants RGPIN-2018-05705 and RGPAS-2018-522715).}
}

\author{%
  \IEEEauthorblockN{%
    Ziteng Cheng\IEEEauthorrefmark{1}\IEEEauthorrefmark{4},
    Sebastian Jaimungal\IEEEauthorrefmark{1}\IEEEauthorrefmark{4} and
    Nick Martin\IEEEauthorrefmark{3}\IEEEauthorrefmark{4} 
  }%
  \thanks{\IEEEauthorrefmark{1} Deptartment of Statistical Sciences, University of Toronto, Canada \{sebastian.jaimungal, ziteng.cheng\}@utoronto.ca}
  \thanks{\IEEEauthorrefmark{3} nick.martin@mail.utoronto.ca}
  \thanks{\IEEEauthorrefmark{4} Equal contribution}
}

\maketitle

\begin{abstract}
We introduce a distributional method for learning the optimal policy in risk averse Markov decision process with finite state action spaces, latent costs, and stationary dynamics. We assume sequential observations of states, actions, and costs and assess the performance of a policy using dynamic risk measures constructed from nested Kusuoka-type conditional risk mappings. For such performance criteria, randomized policies may outperform deterministic policies, therefore, the candidate policies lie in the d-dimensional simplex where d is the cardinality of the action space. Existing risk averse reinforcement learning methods seldom concern randomized policies, naive extensions to current setting suffer from the curse of dimensionality. By exploiting certain structures embedded in the corresponding dynamic programming principle, we propose a distributional learning method for seeking the optimal policy. The conditional distribution of the value function is casted into a specific type of function, which is chosen with in mind the ease of risk averse optimization. We use a deep neural network to approximate said function, illustrate that the proposed method avoids the curse of dimensionality in the exploration phase, and explore the method’s performance with a wide range of model parameters that are picked randomly.
\end{abstract}


\section{Introduction}
Markov Decision Processes (MDPs) are a type of discrete-time stochastic control problem used for sequential decision-making in situations where costs are partially random and partially under the control of a decision maker. In risk-averse MDPs, the decision maker is concerned with the risk or variability of the outcomes beyond the expected costs. One way to incorporate risk aversion into MDPs is to use nested compositions of risk transition mappings. This approach ensures the time-consistency property and ultimately enables the use of a dynamic programming principle (DPP) to solve the corresponding sequential optimization problem. The approach is proposed in \cite{Ruszczynski2010Risk}, where deterministic costs are considered. Both finite and infinite (required bounded costs) time horizon DPP are derived. Subsequent studies such as \cite{Shen2014Risk} and \cite{Chu2014Markov} explore infinite time horizon risk-averse DPPs with unbounded costs in different settings. \cite{Bauerle2021Markov} considers unbounded latent costs and established the corresponding finite and infinite horizon DPPs. More recently, \cite{Cheng2022Markov} has developed a framework based on Kusuoka-type conditional risk mappings that also takes into account randomized actions in a risk averse manner. Depending on the type of conditional risk mappings used, randomized actions may be more preferable than deterministic actions, as illustrated in a motivating example in \cite{Cheng2022Markov}. Other methods of incorporating risk aversion into MDPs are also available, including those discussed in \cite{Bauerle2013Markov}, \cite{Chu2014Markov}, \cite{Bielecki2021Risk}, and the references therein.

The focus of this paper is on the approach of nested compositions of risk transition mappings. The main objective is to develop a reinforcement learning method that solves the infinite horizon risk-averse MDP problem presented in \cite{Cheng2022Markov}. Specifically, the aim is to solve this problem with finite state and action spaces, deterministic latent costs, and stationary dynamics, without assuming knowledge of the controlled transition matrix or cost function. We begin by briefly reviewing some algorithms that solve infinite horizon risk-averse MDP problems. 

For instance, \cite{tamar2015policy} derives a policy gradient formula by combining the static gradient formula for coherent risk measure with the corresponding DPP. This approach is further developed in a sample-based method in \cite{Tamar2017Sequential}, with its convergence analyzed in \cite{Huang2021Convergence}. \cite{Yu2018Approximate} proposes a family of sample-based algorithms to approximately solve problems with continuous state and action spaces. \cite{Shen2014Riskb} presents and analyzes a risk-averse Q-learning algorithm, while \cite{Huang2017Risk} extends the previous Q-learning algorithm based on estimating a general minimax function with stochastic approximation, with detailed error analysis conducted in \cite{Huang2021Stochastic}. \cite{Kose2021Risk} studies a risk-averse temporal difference method that evaluates the value function using linear function approximations. Finally, the recent work in \cite{Coache2022Reinforcement} develops an approach to address risk transition mappings induced by convex risk measures. 

However, the methods mentioned above do not directly apply to the problem presented in \cite{Cheng2022Markov}, where the risk aversion also involves the randomness in the randomized actions. This is mainly because of the lack of linearity: the value function of a randomized action may not be a linear combination of the value functions of individual actions with respect to the randomizing action kernel. Naively extending the existing methods may result in a situation where we need to learn the value functions for numerous pairs of states and action kernels. Since the admissible action kernels form a $d$-dimensional simplex, where $d$ is the size of the action space, the exploration task that follows may suffer from the curse of dimensionality and demand a significant amount of data. On the other hand, the finite nature of the underlying state and action spaces suggests that we can avoid such an excessively expensive exploration task.

We propose a distributional method to address the challenges posed by the risk aversion towards randomized actions in the problem presented in \cite{Cheng2022Markov}. The proposed method learns an auxiliary function that contains sufficient information about the value function's distribution, avoiding the curse of dimensionality and facilitating the computation of the value function defined via a risk transition mapping. We show in Theorem \ref{thm:main} that the proposed method's exploration effort grows polynomially with the state and action space cardinalities. Although we initially considered deterministic latent costs, our method naturally handles random costs whose distribution depends on the current state, the realized action, and the next state. This type of random cost is seldom considered in existing literature on risk-averse reinforcement learning. We provide numerical examples that demonstrate the efficacy of the proposed method at the end of this report.

Using distributional methods to solve MDP problems that are not risk neutral has a long history (cf.\cite{Jaquette1973Markov}, \cite{White1988Mean}, \cite{Morimura2010Nonparametric}, and the reference therein). More recently, a series of works including \cite{Bellemare2017Distributional}, \cite{Dabney2018Implicit}, \cite{Yang2019Fully}, \cite{Nguyen-Tang2021Nguyen-Tang}, and \cite{Zhou2021Nondecreasing} have demonstrated that distributional methods can also achieve better results in the risk neutral setting. In this broader context, our method also contributes to the understanding of the capabilities of distributional methods in solving MDP problems.

\section{Preliminaries}
In this section, we present the set up the this paper.

\subsection{Markov decision process}
Let $(\Omega,\sF,\bP)$ be a probability space. We consider a time-homogeneous Markov decision process (MDP) with a finite state space $\bX$ and finite action space $\bA$. For each $k\in\bA$, let $T^k\in\bR^{|\bX|\times|\bX|}$ be a controlled transition matrix, where $T^k_{ij}$ is the probability of transitioning to state $j\in\bX$ at the next epoch, given the current state $i\in\bX$ and action $k\in\bA$. Let $\pi:\bX\to\cP(\bA)$ be a stationary Markovian policy, where $\cP(\bA)$ is the set of probability measures on $\bA$. Since $\bA$ is finite, $\cP(\bA)$ is a $|\bA|$-dimensional simplex, and for $\lambda\in\cP(\bA)$, $\lambda_k$ is the probability of action $k$ occurring. The state-action process subject to policy $\pi$ is denoted by $\set{(X^\pi_t, A^\pi_t)}_{t\ge 0}$. The MDP is associated with a bounded latent cost function $C:\bX\times\bA\times\bX\to[0,c_{\text{max}}]$, where $c_{\text{max}}>0$ is the upper bound of the cost. Finally, we let $\gamma\in(0,1)$ be the discount factor. 

\subsection{Risk averse dynamic programming}
In this paper, we use the notation $L^\infty(\Omega,\sF,\bP)$ to denote the space of bounded real-valued Borel-measurable random variables.  Equality and inequality between random variables are understood in a $\bP$-almost sure sense.  Let $\sG\subseteq\sF$ be a $\sigma$-algebra.  We say $\zeta:L^\infty(\Omega,\sF,\bP)\to L^\infty(\Omega,\sG,\bP)$ is a conditional risk mapping if $\zeta$ satisfies the following conditions for any $Z, Z^1,Z^2\in L^\infty(\Omega,\sF,\bP)$, $Y\in L^\infty(\Omega,\sG,\bP)$ and $\beta\ge 0$,
\begin{itemize}
\item[(i)][Monotonicity] if $Z^1\le Z^2$, then $\zeta(Z^1)\le\zeta(Z^2);$
\item[(ii)][Translation equivariance] $\zeta(Y+Z) = Y + \zeta(Z);$
\item[(iii)][Convexity] if $\beta\in[0,1]$, then $\zeta(\beta Z^1+(1-\beta) Z^2) \le \beta\zeta(Z^1) + (1-\beta)\zeta(Z^2);$
\item[(iv)][Positive homogeneity] $\zeta(\beta Z) = \beta\zeta(Z).$
\end{itemize} 
Some may replace $\beta$ in condition (iii) (resp. (iv)) with $Y\in[0,1]$ (resp. $Y\ge 0$), which, for the most part, does not affects the developing of the theory.  

Consider $\sU_0=\set{\emptyset,\Omega}\subseteq\sU_1\subseteq\dots\subseteq\sF$, $\fZ=(Z_t)_{t\in\bN}\subset L^\infty(\Omega,\sF,\bP)$ and $\rho_t:L^\infty(\Omega,\sF,\bP)\to L^\infty(\Omega,\sU_{t},\bP)$. Suppose $|Z_t|\le c_{\text{max}}$ for all $t\in\bN$. \cite{Ruszczynski2010Risk} proposes to use a dynamic risk measure of the form
\begin{gather}
\rho_{t,T}\left(\fZ\right) :=
\begin{cases}
\rho_{t}\left(Z_t + \gamma\rho_{t+1,T}\left(\fZ\right)\right), & t<T,\\
\rho_{T}(Z_T), & t=T,
\end{cases}\nonumber\\
\rho_{0,\infty}\left(\fZ\right) := \lim_{T\to\infty}\rho_{0,T}\left(\fZ\right),\label{eq:rho0infty}
\end{gather}
for MDP optimization problem. It can be shown that the construction above guarantees time consistency of $(\rho_{t,\infty})_{t\in\bN}$.\footnote{We do not adopt verbatim the setting from \cite{Ruszczynski2010Risk} for the sake of smooth transition.}

In what follows, we let $\sU^\pi_0$ be the trivial $\sigma$-algebra, $\sU^\pi_t$ be the $\sigma$-algebra generated by $(X^\pi_1,A^\pi_1,\dots,X^\pi_{t-1},A^\pi_{t-1},X^\pi_t)$, and $\cM$ be a set of discrete probability measures with support contained by $(0,1]$. We consider a specific type of conditional risk mapping
\begin{multline}\label{eq:CondRM}
\rho^\pi_t(Z) := \sup_{\mu\in\cM} \bigg\{\int_{0}^1 \inf_{q\in\bR}\bigg\{ q + \\ \xi^{-1} \int_{\bR} (z-q)_+\,P^{Z|\sU^{\pi}_t}(\dif z) \bigg\}\,\mu(\dif\xi)  \bigg\},
\end{multline}
where the right hand side is inspired by Kusuoka representation of law-invariant coherent risk measure (cf. \cite{Kusuoka2000Law}, \cite[Section 6]{Shapiro2021book}). We note that
\begin{align*}
\inf_{q\in\bR}\left\{ q + \xi^{-1} \int_{\bR} (z-q)_+\,P^{Z|\sU^{\pi}_t}(\dif z) \right\}
\end{align*}
is the conditional version of $\avar_\xi$ under the conditional distribution of $Z$ given $\sU^\pi_t$. The main goal of this paper is to develop a sample-based algorithm that solves the following infinite horizon risk averse MDP optimization problem
\begin{align}\label{eq:Main}
\inf_{\pi}\rho^\pi_{0,\infty}\left(\left(C(X^\pi_t,A^\pi_t,X^\pi_{t+1})\right)_{t\in\bN}\right),
\end{align}
where $\rho^\pi_{0,\infty}$ is defined analogously to \eqref{eq:rho0infty}.

The problem \eqref{eq:Main} can be solved using a dynamic programming principle. Specifically, we let $S$ be the Bellman operator acting on $v:\bX\to\bR$ defined as 
\begin{multline}\label{eq:DefS}
S v(i) := \inf_{\lambda\in\cP(\bA)}\sup_{\mu\in\cM} \bigg\{ \int_{0}^1 \inf_{q\in\bR}  \bigg\{  q +  \\
\xi^{-1}\sum_{k\in\bA}\lambda_k \sum_{j\in\bX} T^k_{ij}\Big(C(i,k,j) + \gamma v(j)-q\Big)_+ \bigg\} \mu(\dif\xi) \bigg\}.
\end{multline}
We can restrict $v$ to take values in $[0,\frac{c_{\text{max}}}{1-\gamma}]$ due to the boundedness of the cost. This allows us to replace $q\in\bR$ in \eqref{eq:DefS} with $q\in[0,\frac{c_{\text{max}}}{1-\gamma}]$. It can be shown that $S$ is a $\gamma$-contraction, and the fixed point of $S$, denoted by $v^*$, is the optimal value function. If $\pi^*:\bX\to\cP(\bA)$ attains the infimum in $Sv^*(i)$ for all $i\in\bX$, then $\pi^*$ is the optimal stationary policy, and in fact, it is also optimal among all history-dependent policies. We refer to \cite{Cheng2022Markov} for more discussion in a general setting.

\section{Distributional method for risk-averse learning}
In this section, we introduce a novel concept called $g$-values. We then propose a learning method based on $g$-values and establish a convergence result under suitable conditions, as stated in Theorem \ref{thm:main}. Finally, we provide a detailed description of the algorithm for implementing the method.
\subsection{$g$-value}
In view of \eqref{eq:DefS}, we define the $Q$-value as
\begin{multline}
Q(i,\lambda) := \sup_{\mu\in\cM}\int_{[0,1]}\inf_{q\in\bR}\bigg\{ q + \xi^{-1} \sum_{k\in\bA}\lambda_k \\ \sum_{j\in\bX}T^{k}_{ij}\,\Big(C(i,k,j) + \gamma v^*(j) - q\Big)_+ \bigg\}\,\mu(\dif\xi),
\end{multline}
and derive the following equation for $Q$-learning
\begin{multline}\label{eq:Q}
Q(i,\lambda) = \sup_{\mu\in\cM}\int_{[0,1]}\inf_{q\in\bR}\bigg\{ q + \xi^{-1} \sum_{k\in\bA}\lambda_kx \\ \sum_{j\in\bX}\,\Big(C(i,k,j) + \gamma \inf_{\lambda\in\cP(\bA)}Q(j,\lambda) - q\Big)_+ \bigg\}\,\mu(\dif\xi).
\end{multline}
However, learning the $Q$ function on a fine grid of $\bX\times\cP(\bA)$ turns out to be excessively expensive. Therefore, instead of continuing with \eqref{eq:Q}, we propose to learn the following $g$-value\footnote{By using $g$ to express trapezoidal shaped functions and invoking dominated convergence (cf. \cite[Theorem 2.8.1]{Bogachev2006book}) and monotone class theorem (cf. \cite[Theorem 1.9.3 (ii)]{Bogachev2006book}), it can be shown that the function $q\mapsto\bE((Z-q)_+)$, $q\in\bR$, characterizes the distribution of $Z$.}\footnote{One may derive an equation for $g$-value in analogous to \eqref{eq:Q}, but such equation needs not leads to a contraction in general.}
\begin{align}\label{eq:Defg}
g(i,\lambda,q):=\sum_{k\in\bA}\lambda_k\sum_{j\in\bX}T^{k}_{ij}\,\big(C(i,k,j) + \gamma v^*(j) - q\big)_+.
\end{align}
Such $g$-value has an advantage of being linear in $\lambda$, which helps mitigates the cost of exploration. Moreover, it is worth noting that $q\mapsto g(i,\lambda,q)$ is non-increasing and $1$-Lipschitz for any $(i,\lambda)\in\bX\times\cP(\bA)$, which will be useful in future analysis. Furthermore, we argue that $g$-value is aligned with our goal of solving \eqref{eq:Main}, since by the aforementioned DPP and \eqref{eq:Defg}, we have 
\begin{multline*}
v^*(i) = \inf_{\lambda\in\cP(\bA)}\sup_{\mu\in\cM} \\ \bigg\{ \int_{(0,1]}    \inf_{q\in[0,\frac{c_{\text{max}}}{1-\gamma}]}\bigg\{  q + \xi^{-1}g(i,\lambda,q) \bigg\}\mu(\dif\xi)\bigg\}.
\end{multline*}
This formula shows that the $g$-value is a crucial ingredient in our approach for solving \eqref{eq:Main}.

\subsection{Theoretical foundation}
Suppose that we have observed the running states, actions and costs subject to some exploration policy upto time $t_{\text{max}}$, resulting in a set of data $\set{(x_t,a_t,x_{t+1}, c_t)}_{t=1}^{t_{\text{max}}-1}$, where $c_t = C(x_t,a_t,x_{t+1})$. In order to approximate $g$, we employ a parameterized model $f_\theta:\bX\times\bA\times\bR\to\bR$, where $\theta\in\Theta$ is the parameter, and $f_\theta(i,k,q)$ is designated to approximate $g(i,\delta_k,q)$, where $\delta_k$ is the Dirac measure on $k$. In view of \eqref{eq:Defg}, $g(i,\lambda,q)$ can be approximated by $\sum_{k\in\bA}\lambda_k f_\theta(i,k,q)$. We use $\hat\theta$ to denote the estimate of the optimal parameter (if exists). In view of the bounded cost and positive discount factor, we approximate $v^*$ with $\hat v:\bX\to[0,\frac{c_{\text{max}}}{1-\gamma}]$. Heuristically, we want to update $\hat\theta$ and $\hat v$ recursively in the following way
\begin{align}\label{eq:ExactUpdate}
\begin{cases}
\begin{multlined}
\hat\theta_{n+1} \in \argmin_{\theta\in\Theta} \sup_{(i,k,q)\in\bX\times\bA\times[0,\frac{c_{\text{max}}}{1-\gamma}]}\\[-1.2em] \bigg(f_\theta(i,k,q)  - \sum_{j\in\bX}\hat T^k_{ij}\big(C(i,k,j)+\gamma\hat v_n(j)-q\big)_+\bigg)^2,
\end{multlined}\\
\begin{multlined}
\hat v_{n+1}(i) = \inf_{\lambda\in\cP(\bA)}\sup_{\mu\in\cM} \int_{(0,1]}  \\[-1.2em] 
\inf_{q\in[0,\frac{c_{\text{max}}}{1-\gamma}]} \bigg\{ q +v\xi^{-1}\sum_{k\in\bA}\lambda_k f_{\hat\theta_{n+1}}(i,k,q) \bigg\}\mu(\dif\xi),   
\end{multlined}
\end{cases}
\end{align}
where we define
\begin{align*}
\hat T^k_{ij} := \frac{\sum_{t=1}^{t_{\text{max}}-1}\1_{(i,k,j)}(x_t,a_t,x_{t+1})}{\sum_{t=1}^{t_{\max}-1}\1_{(i,k)}(x_t,a_t)}.
\end{align*}
It is well-known that $\hat T^k_{ij}$ is the MLE of the transition probability (cf. \cite{Anderson1957Statistical}). Note that $\sum_{j\in\bX}\hat T^k_{ij}\big(c_t+\gamma\hat v(x_{t+1})-q\big)_+$ is a convex and $1$-Lipschitz function of $q$ that falls within the range $[0,\frac{c_{\text{max}}}{1-\gamma}]$. Therefore, although the objective involves the supremum over an uncountable set, updating $\hat{\theta}$ is not infeasible. However, such an update requires knowledge of $C$. To circumvent this requirement, we observe that for $q$ fixed, 
\begin{multline}\label{eq:ImpObs}
\sum_{t=1}^{T-1} \big(y_{x_t a_t} - (c_t+\gamma\hat v(x_{t+1})-q)_+\big)^2 \\
=  \sum_{(i,k)\in\bX\times\bA}\sum_{t=1}^{T-1} \1_{(i,k)}(x_t,a_t)\sum_{j\in\bX}\1_{j}(x_{t+1}) \\  \big(y_{ik} - (C(i,k,j)+\gamma\hat v(j)-q)_+\big)^2,
\end{multline}
as a function of $(y_{ik})_{(i,k)\in\bX\times\bA}$, attains the infimum if 
\begin{align}\label{eq:yOpt}
y_{ik} 
= \sum_{j\in\bX}\hat T^k_{ij}\big(C(i,k,j)+\gamma\hat v(j)-q\big)_+,\quad (i,k)\in\bX\times\bA.
\end{align}
We can then use the following updating scheme as an alternative
\begin{align}\label{eq:ExactUpdate2}
\begin{cases}
\begin{multlined}
\hat\theta_{n+1} \in \argmin_{\theta\in\Theta} \sup_{q\in[0,\frac{c_{\max}}{1-\gamma}]} \\[-1.2em] \sum_{t=1}^{T-1}\Big(f_\theta(x_t,a_t,q)  - \big(c_t+\gamma\hat v_n(x_{t+1})-q\big)_+ \Big)^2,
\end{multlined}
\\
\begin{multlined}
\hat v_{n+1}(i) = \inf_{\lambda\in\cP(\bA)}\sup_{\mu\in\cM} \\[-1.2em]  \int_{(0,1]} \inf_{q\in[0,\frac{c_{\text{max}}}{1-\gamma}]}\bigg\{  q + \xi^{-1}\sum_{k\in\bA}\lambda_k f_{\hat\theta_{n+1}}(i,a,q) \bigg\}\mu(\dif\xi).
\end{multlined}
\end{cases}
\end{align}

In order to obtain a convergence result, we make the following technical assumption. 

\begin{assumption}\label{assump}
Let $c_{\text{max}}>0$, $\ell\in\bN$, $b,\varepsilon_e\in(0,1)$, $\varepsilon_\theta,\varepsilon_v>0$ be absolute constants. We assume that
\begin{itemize}
\item[(i)] the range of the cost function $C$ is contained by $[0,c_{\text{max}}]$;
\item[(ii)] $\sup_{\mu\in\cM}\mu([0,b])=0$;
\item[(iii)] $\{(X^\pi_t,A^\pi_t)\}_{t=1}^T$ is subject to an exploration policy $\pi$ such that 
\begin{align*}
\bP\bigg(\sum_{r=t+1}^{t+\ell} \1_{(i,k)}(X^\pi_{r},A^\pi_{r})\ge 1\bigg|\cF^{\pi}_t\bigg) > \varepsilon_e, 
\end{align*}
for any $(t,i,k)\in \bN\times\bX\times\bA$, where $\cF^{\pi}_t:=\sigma(X^\pi_1,A^\pi_1,\dots,X^\pi_t,A^\pi_t)$;
\item[(iv)] regardless of the data and $\hat v:\bX\to[0,\frac{c_{\text{max}}}{1-\gamma}]$, we always find $\hat\theta_{\text{new}}\in\Theta$ and $\hat v_{\text{new}}:\bX\to[0,\frac{c_{\text{max}}}{1-\gamma}]$ such that, for all $(i,k,q)\in\bX\times\bA\times[0,\frac{c_{\max}}{1-\gamma}]$,
\begin{multline}
\bigg|f_{\hat\theta_{\text{new}}}(i,k,q) - \sum_{j\in\bX}\hat T^k_{ij}\big(C(i,k,j)+\gamma\hat v(x_{t+1})-q\big)_+\bigg|\\  \le \varepsilon_\theta,\label{eq:thetaApproxUpdate}
\end{multline}
and
\begin{multline}
\sup_{i\in\bX}\bigg| \hat v_{\text{new}}(i) - \inf_{\lambda\in\cP(\bA)}\sup_{\mu\in\cM} \int_0^1 \\    \inf_{q\in[0,\frac{c_{\text{max}}}{1-\gamma}]}\bigg\{  q + \xi^{-1}\sum_{k\in\bA}\lambda_k f_{\hat\theta_{\text{new}}}(i,k,q) \bigg\}\mu(\dif\xi) \bigg| \le \varepsilon_{v}.\label{eq:vApproxUpdate} 
\end{multline}

\end{itemize}
\end{assumption}
Condition (i) and (ii) follows automatically from the setting above; these conditions are included in the assumption for the sake of easy navigation. Condition (iii) is a version of parallel sampling model (PSM). PSM was originally introduced in \cite{Kearns1999Finite} and is commonly used in reinforcement learning literature as an exploration policy that achieves perfect exploration (cf. \cite{Even-Dar2003Finite}). Condition (iv) regards the accuracy of the update. In particular, \eqref{eq:thetaApproxUpdate} corresponds to the computation of $\hat\theta_{n+1}$ in \eqref{eq:ExactUpdate2}. Based on the separability of the objective illustrated in \eqref{eq:ImpObs}, the convexity of $\big(y_{ik} - (C(i,k,j)+\gamma\hat v(j)-q)_+\big)^2$ in $y_{ik}$, and the observed good behavior of $q\mapsto\sum_{j\in\bX}\hat T^k_{ij}\big(c_t+\gamma\hat v(x_{t+1})-q\big)_+$, we consider \eqref{eq:thetaApproxUpdate} reasonable.

Below is our main result. The proof is deferred to the appendix.
\begin{theorem}\label{thm:main}
Suppose Assumption \ref{assump}. Let $t_{\max}>\ell$. Given data $\set{(x_t,a_t,x_{t+1}, c_t)}_{t=1}^{t{\max}-1}$ and an arbitrary $\hat v_0:\bX\to[0,\frac{c_{\text{max}}}{1-\gamma}]$, we compute $\set{(\hat\theta_{n},\hat v_n)}_{n\in\bN}$ according to \eqref{eq:ExactUpdate2}, approximately as in Assumption \ref{assump} (iv). Then, for any $\varepsilon\in(0,1]$, there is a probability of at least
\begin{align*}
1 - 3|\bX|^2 |\bA|\bigg( e^{ - \frac{\varepsilon_e^2}4 \lfloor\frac{t_{\max}-1}{\ell}\rfloor } + e^{-\frac{\varepsilon^2\varepsilon_e^2 }{8\ell}\lfloor\frac{t_{\max}-1}{\ell}\rfloor} \bigg)
\end{align*}
that 
\begin{align*}
\|\hat v_{n} - v^*\|_\infty \le \gamma^n\|\hat v_0-v^*\|_\infty + \frac{b^{-1} c_{\max} \varepsilon}{(1-\gamma)^2} + \frac{b^{-1}\varepsilon_\theta+\varepsilon_v}{1-\gamma}.
\end{align*}
for all $n\in\bN$.
\end{theorem}
Sometimes it is advisable to assume that $\varepsilon_e$ is proportional to $(|\bX||\bA|)^{-1}$. In order to maintain the same level of accuracy (in terms of the probability bound), we need to set $t_{\max}\propto|\bX|^2|\bA|^2 \log(|\bX|^2|\bA|)$. In this case, the effort required for exploration only needs to grow polynomially as $|\bX||\bA|$ increases.

\subsection{Implementation }
In our algorithm, we use a deep neural network for $f_\theta$. 
We let $m_{\text{grid}}\in\bN$ and $(q_1,\dots,q_{m_\text{grid}})$ be a pre-selected grid on $[0,\frac{c_{\max}}{1-\gamma}]$. We are given data $\set{(x_t,a_t,x_{t+1}, c_t)}_{t=1}^{t_{\text{max}}-1}$, and an a priori guess $\hat v$ of the value function. 

Instead of following strictly \eqref{eq:ExactUpdate2}, we consider the updating procedure below
\begin{multline}\label{eq:approxtheta}
\hat\theta_{n+1} \in \argmin_{\theta\in\Theta} \\ \sum_{m=1}^{m_{\text{grid}}} \sum_{t=1}^{T-1}\bigg(f_\theta(x_t,a_t,q_m)  - \Big(c_t+\gamma\hat v_n(x_{t+1})-q_m\Big)_+ \bigg)^2\\
+ \beta\psi(\theta),
\end{multline}
where $\psi$ is a penalization for ensuring monotonicity on $q\mapsto f_\theta(i,k,q)$, defined as 
\begin{align*}
\psi(\theta) := \sum_{(i,k)\in\bX\times\bA}\sum_{m=1}^{m_{\text{grid}}-1} \big(f_\theta(i,k,q_{m+1}) - f_\theta(i,k,q_{m})\big)_+,
\end{align*}
and $\beta\ge 0$ is the regularization parameter. We use stochastic gradient descent for  $\argmin_{\theta\in\Theta}$. After \eqref{eq:approxtheta} is done, we perform 
\begin{multline}\label{eq:approxv}
\hat v_{n+1}(i) = \inf_{\lambda\in\cP(\bA)}\max_{\mu\in\cM} \\  \int_{(0,1]} \inf_{q\in[0,\frac{c_{\text{max}}}{1-\gamma}]}\bigg\{  q + \xi^{-1}\sum_{k\in\bA}\lambda_k f_{\hat\theta_{n+1}}(i,a,q) \bigg\}\mu(\dif\xi),
\end{multline}
where we recall that $\cM$ is a finite set of discrete probabilities on $[0,1]$, and thus the $\int_{(0,1]}$ is in fact a finite sum. We use gradient descent with random initialization for $\inf_{q\in[0,\frac{c_{\text{max}}}{1-\gamma}]}$, and random search for $\inf_{\lambda\in\cP(\bA)}$. After \eqref{eq:approxv} is done, we may return to \eqref{eq:approxtheta} for next round of update. In order to obtain an approximated optimal policy $\hat\pi$, we should record the approximated minimizor of $\inf_{\lambda\in\cP(\bA)}$ for each $i$.

We summarize the implementation in Algorithm \eqref{algo}. We point out that, Algorithm \eqref{algo} can be integrate asynchronously into a larger implementation that involves running data.

 \begin{algorithm}\label{algo}
 \caption{Distributional Method for Risk-Averse RL}
 \begin{algorithmic}[1]
 \renewcommand{\algorithmicrequire}{\textbf{Input:}}
 \renewcommand{\algorithmicensure}{\textbf{Output:}}
 \REQUIRE data $\set{(x_t,a_t,x_{t+1}, c_t)}_{t=1}^{t_{\text{max}}-1}$, aprior guess $\hat v$
 \ENSURE  $\hat v$, $\hat \pi$
 \REPEAT
  \STATE Randomly initialize $\theta$, or use previous $\theta$ if given
  \WHILE{find $\theta$}
  \STATE Stochastic gradient descent according to \eqref{eq:approxtheta}
  \ENDWHILE
  \WHILE{update $\hat v$ and $\hat\pi$}
  \STATE Random search according to \eqref{eq:approxv} \\
  Record the best $\lambda\in\cP(\bX)$ for each $i\in\bX$
  \ENDWHILE 
 \UNTIL{Convergence}
   \STATE Update $\hat\pi(i)$ according to the record 
 \RETURN $\hat v$, $\hat\pi$ 
 \end{algorithmic} 
 \end{algorithm}

\section{Numeric experiments}
In this section, we present numerical experiments to validate the performance of Algorithm \ref{algo}. We consider a state-action space with $|\bX|=|\bA|=4$, and a discount factor of $\gamma=0.3$. We use the following $\cM$ for the conditional risk mapping \eqref{eq:CondRM}:
\begin{multline*}
\cM = \big\{0.2\delta_{0.2} + 0.8\delta_{1}, \delta_{0.5}, 0.1\delta_{0.05}+0.5\delta_{0.4}+0.6\delta_{0.6},\\
0.5\delta_{0.3}+0.5\delta_{0.8}\big\},
\end{multline*}
where $\delta$ denotes the Dirac measure. The transition matrices used in the experiment are randomly generated. Although our algorithm was introduced for deterministic latent costs, it also handles random costs without requiring significant modifications. We test the algorithm with various random costs, such as $\text{Beta}(\alpha,\beta)$ with $\alpha,\beta:\bX\times\bA\times\bX\to(0,\infty)$ depending on the current state, the realized action, and next state. We assume the knowledge of $[0,c_{\max}]$, and set $(q_1,\dots,q_{m_{\text{grid}}})$ as a uniform partition of $[0,c_{\max}]$ with $m_{\text{grid}}=100$. We set $t_{\text{max}}=10000$ and sample according to some randomly picked stationary policy. We then compute $\hat v$ and $\hat \pi$ using Algorithm \ref{algo}. To ensure accuracy when updating $\hat v$ and $\hat\pi$, we perform a thorough random search. However, we conjecture that there is a certain structure that we can take advantage of in learning $\hat v$ and $\hat \pi$, and the computation cost does not grow exponentially as $|\bA|$ increases. In Figure \ref{fig}, we plot the relative errors of $\hat v$ for each $i\in\bX$ in 10 different experiments. The benchmark in each experiment is computed using brute force search.

\begin{figure}[htbp]
\includegraphics[width=1\columnwidth]{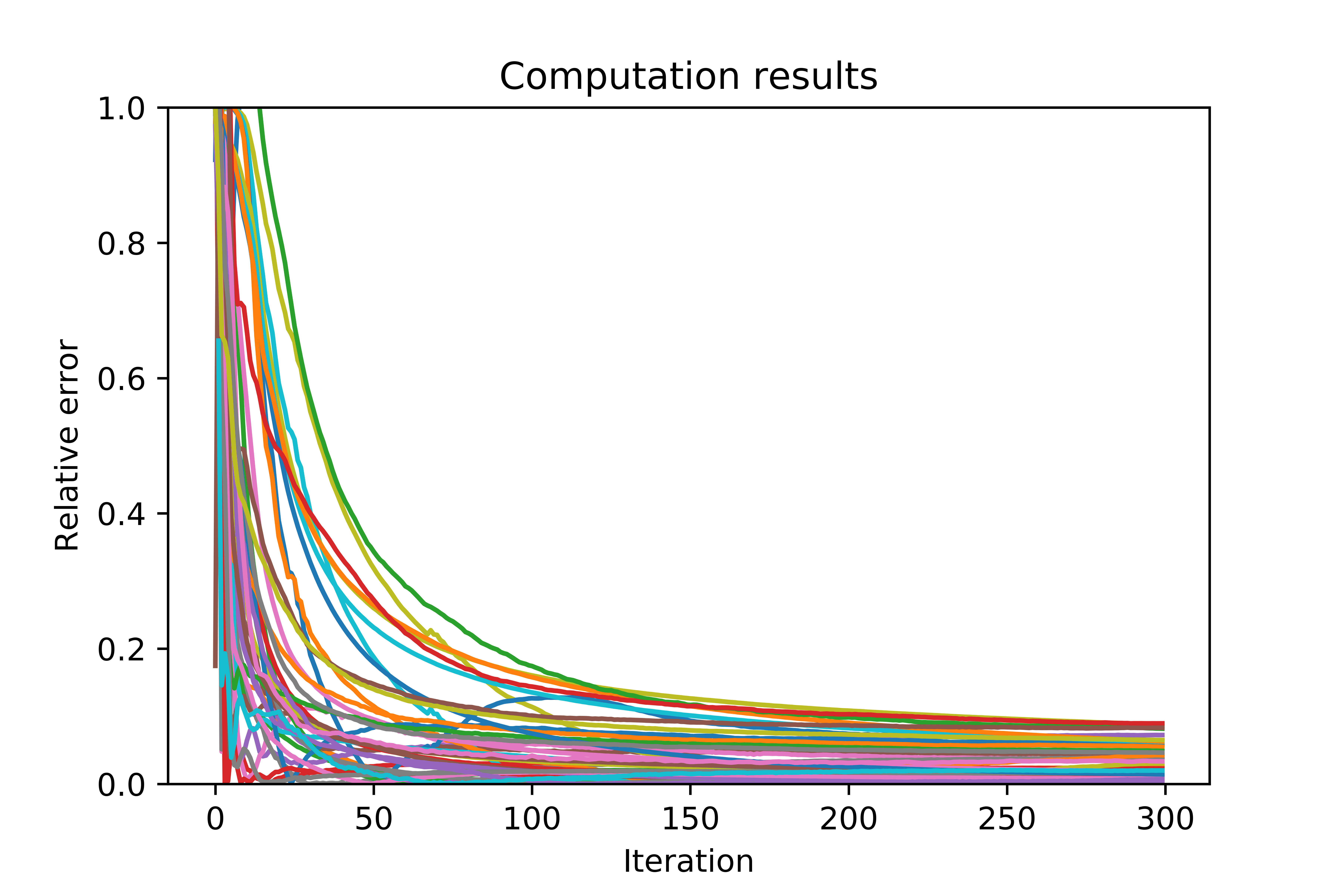}
\caption{Relative errors of all value functions in 10 experiments}
\label{fig}
\end{figure}








\bibliographystyle{IEEEtranS} 
\bibliography{refs}
\vspace{12pt}

\onecolumn

\appendices

\section{Proof of Theorem \ref{thm:main}}

We fix $t_{\text{max}}>1$ and $\pi:\bX\to\cP(\bA)$ for the remainder of this section. Firstly, we will introduce the contraction property of $S$.
\begin{lemma}\label{lem:Contraction}
For any $v,v':\bX\to\bR$, $\|Sv-Sv'\|_\infty\le \gamma\|v-v'\|_\infty$.
\end{lemma}
\begin{proof}
This is an immediate consequence of \cite[Lemma 3.3]{Cheng2022Markov}.
\end{proof}

The proof of Theorem \ref{thm:main} is also dependent on the following two technical lemmas.
\begin{lemma}\label{lem:EstP}
For any $(i,j,k)\in\bX\times\bA\times\bX$, $\varepsilon\in(0,1)$ and integer $N<\varepsilon_e\lfloor\frac{t_{\max}-1}{\ell}\rfloor$, we have 
\begin{align*}
\bP\bigg(\bigg|\frac{\sum_{t=1}^{t_{\text{max}}-1}\1_{(i,k,j)}(x_t,a_t,x_{t+1})}{\sum_{t=1}^{t_{\max}-1}\1_{(i,k)}(x_t,a_t)}- T^{k}_{ij}\bigg| > \varepsilon\bigg) \le \exp\bigg( - \frac{(N-\varepsilon_e \lfloor\frac{t_{\max}-1}{\ell}\rfloor)^2} {\lfloor \frac{t_{\max}-1}{\ell}\rfloor} \bigg) + 2\exp\left(-\frac{\varepsilon^2N^2}{2t_{\max}}\right).
\end{align*}
\end{lemma} 
\begin{proof}
To start with note that
\begin{align*}
&\bigg\{\bigg|\frac{\sum_{t=1}^{t_{\text{max}}-1}\1_{(i,k,j)}(X_t,X_t,X_{t+1})}{\sum_{t=1}^{t_{\max}-1}\1_{(i,k)}(X_t,X_t)}- T^{k}_{ij}\bigg| \ge \varepsilon\bigg\}\\
&\quad\subseteq \bigg\{\sum_{r=1}^{t_{\max}-1}\1_{(i,k)}(X^\pi_r,A^\pi_r) < N \bigg\} \cup \bigg( \bigg\{\sum_{r=1}^{t_{\max}-1}\1_{(i,k)}(X^\pi_r,A^\pi_r) \ge N \bigg\}  \cap \bigg\{\bigg|\frac{\sum_{t=1}^{t_{\text{max}}-1}\1_{(i,k,j)}(X_t,X_t,X_{t+1})}{\sum_{t=1}^{t_{\max}-1}\1_{(i,k)}(X_t,X_t)}- T^{k}_{ij}\bigg| \ge \varepsilon\bigg\} \bigg)\\
&\quad\subseteq \bigg\{\sum_{r=1}^{t_{\max}-1}\1_{(i,k)}(X^\pi_r,A^\pi_r) < N \bigg\} \cup \bigg\{\bigg|\sum_{t=1}^{t_{\text{max}}-1}\1_{(i,k,j)}(X_t,X_t,X_{t+1})- T^{k}_{ij}\sum_{t=1}^{t_{\max}-1}\1_{(i,k)}(X_t,X_t)\bigg| \ge \varepsilon N\bigg\}.
\end{align*}
Therefore,
\begin{align}\label{eq:EstP}
&\bP\bigg(\bigg|\frac{\sum_{t=1}^{t_{\text{max}}-1}\1_{(i,k,j)}(x_t,a_t,x_{t+1})}{\sum_{t=1}^{t_{\max}-1}\1_{(i,k)}(x_t,a_t)}- T^{k}_{ij}\bigg| > \varepsilon\bigg)\nonumber\\
&\quad\le \bP\bigg(\sum_{r=1}^{t_{\max}-1}\1_{(i,k)}(X^\pi_r,A^\pi_r) < N \bigg) + \bP\bigg( \bigg|\sum_{t=1}^{t_{\text{max}}-1}\1_{(i,k,j)}(X_t,X_t,X_{t+1})- T^{k}_{ij}\sum_{t=1}^{t_{\max}-1}\1_{(i,k)}(X_t,X_t)\bigg| \ge \varepsilon N \bigg).
\end{align}
In order to investigate the first term in right hand side of \eqref{eq:EstP}, we introduce an auxiliary process. For $\iota=1,\dots,\lfloor \frac{t_{\max}-1}{\ell} \rfloor$, we let
\begin{align*}
L_\iota := \sum_{t=1}^{\ell \iota}\1_{(i,k)}(X_t,A_t) - \varepsilon_e\iota.
\end{align*}
Note that $(L_\iota)_{\iota=1}^{\lfloor \frac{t_{\max}-1}{\ell} \rfloor}$ is a sub-martingale under the filtration $(\sF^\pi_{\ell\iota})_{r=1}^{\lfloor \frac{t_{\max}-1}{\ell} \rfloor}$. Indeed, by the Markov property of $\set{(X^\pi_{t},A^\pi_{t})}_{t\in\bN}$, we have
\begin{align*}
\bE\big(L_{\iota+1} \big| \sF^\pi_{\ell\iota} \big) = L_\iota + \bE\bigg( \sum_{t=\ell \iota+1}^{\ell (\iota+1)}\1_{(i,k)}(X_t,A_t) - \varepsilon_e \bigg| \sF^\pi_{\ell\iota} \bigg)\ge L_\iota,
\end{align*}
where we have used Assumption \ref{assump} (iii) in the last equality. Then, by Azuma's inequality, for $N<\varepsilon_e\lfloor \frac{t_{\max}-1}{\ell} \rfloor$,
\begin{align}\label{eq:EstP1}
\bP\bigg(\sum_{r=1}^{t_{\max}-1}\1_{(i,k)}(X^\pi_r,A^\pi_r) < N \bigg) &\le \bP\bigg( L_{\lfloor \frac{t_{\max}-1}{\ell} \rfloor} \le N-\varepsilon_e\bigg\lfloor \frac{t_{\max}-1}{\ell} \bigg\rfloor\bigg) \nonumber\\
&\le \exp\bigg( - \frac{(N-\varepsilon_e \lfloor\frac{t_{\max}-1}{\ell}\rfloor)^2} {\lfloor \frac{t_{\max}-1}{\ell}\rfloor} \bigg).
\end{align}
Regarding the second term in \eqref{eq:EstP}, we define $M^{ikj}_1:=0$ and 
\begin{align*}
M^{ikj}_t:=\sum_{r=1}^{t-1} \1_{(i,k,j)}(X^\pi_{r},A^\pi_{r},X^{\pi}_{r+1}) - T^{k}_{ij}\sum_{r=1}^{t-1}\1_{(i,k)}(X^\pi_{r},A^\pi_{r}),\quad t\ge 2.
\end{align*}
Note that $(M^{ikj}_t)_{t\in\bN}$ is a $(\cF^\pi_t)_{t\in\bN}$-martingale:
\begin{align*}
\bE\big(M^{ikj}_{t+1}\big|\cF^\pi_t\big) &= M^{ikj}_{t} + \bE\bigg(\1_{(i,k,j)}(X^\pi_{t},A^\pi_{t},X^{\pi}_{t+1}) - T^{k}_{ij}\1_{(i,k)}(X^\pi_{t},A^\pi_{t}) \bigg|\cF^\pi_t\bigg)\\
&= M^{ikj}_{t} + \bE\bigg(\1_{(i,k,j)}(X^\pi_{t},A^\pi_{t},X^{\pi}_{t+1}) - T^{k}_{ij}\1_{(i,k)}(X^\pi_{t},A^\pi_{t}) \bigg|\sigma(X^\pi_{t},A^\pi_{t})\bigg) = M^{ikj}_{t},
\end{align*}
where we have used the Markov property of $\set{(X^\pi_{t},A^\pi_{t})}_{t\in\bN}$ in the second line. It follows from Azuma's inequality that
\begin{align}\label{eq:EstP2}
\bP\big(\big|M^{ijk}_{t_{\max}}\big| \ge \varepsilon N \big) \le \exp\left(-\frac{\varepsilon^2N^2}{2t_{\max}}\right).
\end{align}
Finally, by combining \eqref{eq:EstP}, \eqref{eq:EstP1} and \eqref{eq:EstP2}, we complete the proof.
\end{proof}

\begin{lemma}\label{lem:Estvhatnew}
Let $v^*$ be the fixed point of $S$ defined in \eqref{eq:DefS}. Let $\hat v$ and $\hat v_{\text{new}}$ be introduced as in Assumption \ref{assump} (iv). Then,
\begin{align*}
\|\hat v_{\text{new}}-v^*\|_\infty \le \gamma\|\hat v-v^*\|_\infty + b^{-1}\bigg(\sup_{(i,k,j)\in\bX\times\bA\times\bA}\big|\hat T^{k}_{ij} - T^{k}_{ij}\big|\frac{c_{\max}}{1-\gamma} + \varepsilon_\theta\bigg) + \varepsilon_{v}
\end{align*}
\end{lemma}
\begin{proof}
To start with, by \eqref{eq:vApproxUpdate} and the fact that $v^*=Sv^*$,
\begin{align*}
\|\hat v_{\text{new}}-v^*\|_\infty \le \sup_{i\in\bX}\bigg| \inf_{\lambda\in\cP(\bA)}\sup_{\mu\in\cM} \bigg\{ \int_{(0,1]}    \inf_{q\in[0,\frac{c_{\text{max}}}{1-\gamma}]}\bigg\{  q + \xi^{-1}\sum_{k\in\bA}\lambda_k f_{\hat\theta_{\text{new}}}(i,k,q) \bigg\}\mu(\dif\xi)\bigg\} - Sv^*(i) \bigg| + \varepsilon_v.
\end{align*}
Then, by Assumption \ref{assump} (ii), \eqref{eq:thetaApproxUpdate} and Lemma \ref{lem:Contraction},
\begin{align*}
\|\hat v_{\text{new}}-v^*\|_\infty &\le \sup_{i\in\bX} \bigg| \inf_{\lambda\in\cP(\bA)}\sup_{\mu\in\cM} \bigg\{ \int_{(0,1]}    \inf_{q\in[0,\frac{c_{\text{max}}}{1-\gamma}]}\bigg\{  q + \xi^{-1}\sum_{k\in\bA}\lambda_k\sum_{j\in\bX}\hat T^k_{ij}\big(c_t+\gamma\hat v(x_{t+1})-q\big)_+ \bigg\}\mu(\dif\xi)\bigg\} - Sv^*(i) \bigg| \\
&\quad+ b^{-1}\varepsilon_\theta + \varepsilon_v \\
&\le \sup_{i\in\bX} \bigg| \inf_{\lambda\in\cP(\bA)}\sup_{\mu\in\cM} \bigg\{ \int_{(0,1]}    \inf_{q\in[0,\frac{c_{\text{max}}}{1-\gamma}]}\bigg\{  q + \xi^{-1}\sum_{k\in\bA}\lambda_k\sum_{j\in\bX}\hat T^k_{ij}\big(c_t+\gamma\hat v(x_{t+1})-q\big)_+ \bigg\}\mu(\dif\xi)\bigg\} - S\hat v(i) \bigg| \\
&\quad+ \|S\hat v-Sv^*\|_\infty + b^{-1}\varepsilon_\theta + \varepsilon_v \\
&\le \gamma\|\hat v-v^*\|_\infty + b^{-1}\bigg(\sup_{(i,k,j)\in\bX\times\bA\times\bA}\big|\hat T^{k}_{ij} - T^{k}_{ij}\big|\frac{c_{\max}}{1-\gamma} + \varepsilon_\theta\bigg) + \varepsilon_{v}.
\end{align*}
The proof is complete.
\end{proof}

We are now in position to prove Theorem \ref{thm:main}.
\begin{proof}[Proof of Theorem \ref{thm:main}]
We first simplify Lemma \ref{lem:EstP} by letting $N= \lceil\frac12 \varepsilon_e \lfloor\frac{t_{\max}-1}{\ell}\rfloor\rceil$
\begin{align*}
&\bP\bigg(\bigg|\frac{\sum_{t=1}^{t_{\text{max}}-1}\1_{(i,k,j)}(x_t,a_t,x_{t+1})}{\sum_{t=1}^{t_{\max}-1}\1_{(i,k)}(x_t,a_t)}- T^{k}_{ij}\bigg| > \varepsilon\bigg) \le \exp\bigg( - \frac{\lfloor\frac12 \varepsilon_e \lfloor\frac{t_{\max}-1}{\ell}\rfloor\rfloor^2}{\lfloor \frac{t_{\max}-1}{\ell}\rfloor} \bigg) + 2\exp\left(-\frac{\varepsilon^2\varepsilon_e^2 \lfloor\frac{t_{\max}-1}{\ell}\rfloor^2}{8t_{\max}}\right)\\
&\quad\le e^{\varepsilon_e}\exp\bigg( - \frac{\varepsilon_e^2}4 \bigg\lfloor\frac{t_{\max}-1}{\ell}\bigg\rfloor \bigg) + 2\exp\left(-\frac{\varepsilon^2\varepsilon_e^2 (\frac{t_{\max}}{\ell}-1)\lfloor\frac{t_{\max}-1}{\ell}\rfloor}{8t_{\max}}\right) \\
&\quad \le 3\exp\bigg( - \frac{\varepsilon_e^2}4 \bigg\lfloor\frac{t_{\max}-1}{\ell}\bigg\rfloor \bigg) + 3\exp\left(-\frac{\varepsilon^2\varepsilon_e^2 }{8\ell}\bigg\lfloor\frac{t_{\max}-1}{\ell}\bigg\rfloor\right),
\end{align*}
where we have used the fact that $\lfloor\frac{t_{\max}-1}{\ell}\rfloor \ge \frac{t_{\max}}{\ell}-1\ge 0$ in the second line. Consequently, 
\begin{align*}
&\bP\bigg(\bigg|\frac{\sum_{t=1}^{t_{\text{max}}-1}\1_{(i,k,j)}(x_t,a_t,x_{t+1})}{\sum_{t=1}^{t_{\max}-1}\1_{(i,k)}(x_t,a_t)}- T^{k}_{ij}\bigg| \le \varepsilon, \, (i,k,j)\in\bX\times\bA\times\bX\bigg)\\
&\quad\ge 1- \sum_{(i,j,k)\in\bX\times\bA\times\bX}\bP\bigg(\bigg|\frac{\sum_{t=1}^{t_{\text{max}}-1}\1_{(i,k,j)}(x_t,a_t,x_{t+1})}{\sum_{t=1}^{t_{\max}-1}\1_{(i,k)}(x_t,a_t)}- T^{k}_{ij}\bigg| > \varepsilon, \, (i,k,j)\in\bX\times\bA\times\bX\bigg) \\
&\quad\ge 1 - 3|\bX|^2 |\bA|\bigg( \exp\bigg( - \frac{\varepsilon_e^2}4 \bigg\lfloor\frac{t_{\max}-1}{\ell}\bigg\rfloor \bigg) + \exp\left(-\frac{\varepsilon^2\varepsilon_e^2 }{8\ell}\bigg\lfloor\frac{t_{\max}-1}{\ell}\bigg\rfloor\right) \bigg)
\end{align*}
Finally, under the realization that $\bigg|\frac{\sum_{t=1}^{t_{\text{max}}-1}\1_{(i,k,j)}(x_t,a_t,x_{t+1})}{\sum_{t=1}^{t_{\max}-1}\1_{(i,k)}(x_t,a_t)}- T^{k}_{ij}\bigg| \le \varepsilon$ for all $(i,k,j)\in\bX\times\bA\times\bX$, invoking \eqref{lem:Estvhatnew} iteratively, we yield 
\begin{align*}
\|\hat v_{n}-v^*\|_\infty \le \gamma^n\|\hat v_0-v^*\| + \frac{b^{-1}}{1-\gamma}\bigg(\frac{c_{\max}}{1-\gamma}\varepsilon+\varepsilon_\theta\bigg) + \frac{\varepsilon_v}{1-\gamma},
\end{align*}
which completes the proof.
\end{proof}

\end{document}